\documentclass[letterpaper, 10 pt, conference]{sty_files/ieeeconf}
\usepackage{graphicx}
\usepackage{sty_files/amice_shortcuts}
\usepackage{lipsum}
\usepackage{sty_files/tikz_styles}
\usepackage{bibunits}
\usepackage{svg} 
\usepackage{float}

\usepackage{leftidx}
\usepackage{ragged2e}
\usepackage{cite}
\usepackage{pdfpages}
\usepackage{subcaption}
\usepackage{tabularray}

\usepackage[colorlinks,bookmarksopen,bookmarksnumbered,citecolor=blue,urlcolor=blue]{hyperref}
\hypersetup{ 
    colorlinks=true, 
    linkcolor=black,
    citecolor = black, 
    filecolor=magenta,      
    urlcolor=blue
    }

\usepackage{float}
\usepackage[linesnumbered, ruled, lined, boxed, commentsnumbered]{algorithm2e}
\usepackage{footmisc}
\graphicspath{{figures/}}

\IEEEoverridecommandlockouts                              %
                                                          
\author{Alexandre Amice \and
Peter Werner \and
Russ Tedrake
\thanks{The authors are with the Massachusetts Institute of Technology, {\tt\small\{amice, 
wernerpe,
 russt\}@mit.edu}. This work was supported by Amazon.com; PO\# \#2D-06310236, Air Force Research Laboratory FA8750-19-2-1000
}%
}

\title{Certifying Bimanual RRT Motion Plans in a Second}

\begin{document}

\maketitle
\begin{abstract}
We present an efficient method for certifying non-collision for piecewise-polynomial motion plans in algebraic reparametrizations of configuration space. Such motion plans include those generated by popular randomized methods including RRTs and PRMs, as well as those generated by many methods in trajectory optimization. Based on Sums-of-Squares 
optimization, our method provides exact, rigorous certificates of non-collision; it can never falsely claim that a motion plan containing collisions is collision-free. We demonstrate that our formulation is practical for real world deployment, certifying the safety of a twelve degree of freedom motion plan in just over a second. Moreover, the method is capable of discriminating the safety or lack thereof of two motion plans which differ by only millimeters.
\end{abstract}

\section{Introduction}

Collision-free motion planning is a fundamental problem in the safe and efficient operation of any robotic system. One of the most important subroutines in collision-free motion planning is collision detection, which has been studied extensively in robotics~\cite{cameron1985study, he2014efficient}, computer graphics~\cite{lin1998collision, hubbard1995collision}, and computational geometry more broadly~\cite{shamos1976geometric, lin2017collision}. Algorithms for checking whether a single configuration is collision free are quite mature and can be performed in microseconds on modern hardware \cite{pan2012fcl,tarbouriech2018bisection}.

By contrast, algorithms which are capable of certifying non-collision for the infinite family of configurations along a motion plan are less common. Known as dynamic collision checking, this subroutine is performed hundreds to thousands of times in randomized motion planning methods, and so the speed as well as the correctness of these algorithms is central to their adoption. Due to the speed requirement, it is common practice to heuristically perform dynamic collision checking by sampling a finite number of static configurations along the motion plan. Nevertheless, it is widely recognized that this is insufficient for safety critical robots.

This has motivated a number of more sophisticated algorithms for certifying the safety along a motion plan. These broadly fall into three major families: feature tracking, swept volumes, and trajectory parameterization methods~\cite{jimenez20013d}.

Both the feature tracking and swept volume methods have effectively scaled to applications in both graphics and robotics. The most successful methods rely on pre-computing a hierarchy of bounding volumes at various resolutions and certifying non-intersection at a finite number of points along the motion plan~\cite{jackins1980oct, tang2011ccq}. Interval arithmetic is used to certify the safety in between the static configurations \cite{pan2012collision}. The choice of samples is done in an adaptive manner which guarantees non-collision and some algorithms are fast enough for use as dynamic collision checkers during randomized motion planning~\cite{schwarzer2004exact, tang2011ccq}. 

These methods have two primary drawbacks. The first is that they frequently only consider piecewise linear configuration space motion plans and so cannot generalize to smoother plans. The second drawback can be their conservativeness. Typically, these methods require difficult-to-compute parameters such as an upper bound on the maximum length of a curve along the swept volume~\cite{schwarzer2004exact} or bounding the rate of change of the distance between two objects over the course of the trajectory~\cite{zhang2007continuous}. If these bounds are too loose, collision-free motion plans in tight configuration spaces cannot be certified, which is arguably the most important regime.

On the other hand, trajectory parametrizaton methods can certify any collision-free motion plan to arbitrary precision. This is because these methods make relatively minimal assumptions to obtain exact guarantees, requiring only knowledge of the forward kinematics of the robot and a concrete, description of the collision geometries~\cite{canny1986collision},~\cite{schweikard1991polynomial}. Both are often necessary information for the simulation of any robotic system and so are readily available. However, most of these methods rely on exact, algebraic computations such as polynomial root finding, and so are typically too slow for practical deployment. 
\begin{figure}[t]
    \centering
    \includegraphics[width = 0.7\linewidth]{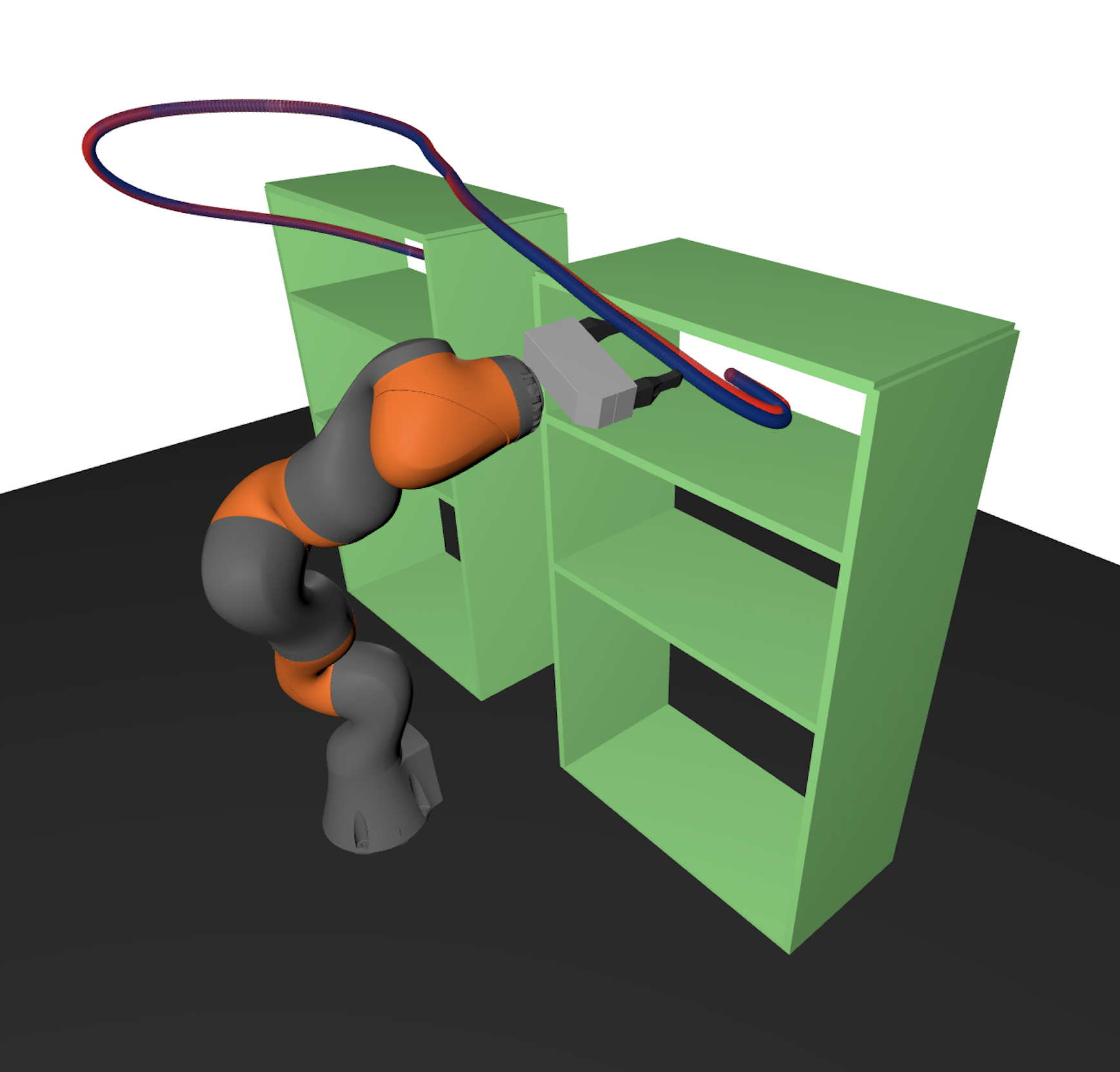}
    \caption{A 7-DOF Kuka iiwa reaching into a pair of shelves. Despite differing by at most 20mm, the blue motion plan is collision-free, while the red motion plan contains minor collision when reaching into each shelf. The proposed certification method is capable of discriminating the safety of these two motion plans.  A video describing the method is available at \url{https://youtu.be/oTiDYeptKis}}
    \vspace{-0.5cm}
    \label{F:iiwa_piecwise_poly}
\end{figure}

\subsection{Contribution}

This paper presents a certification method in the family of trajectory parametrization algorithms which is efficient enough for practical deployment on robotic systems. In particular, we employ sums-of-squares~(SOS) optimization to provide a rigorous method for certifying motion plans parametrized by polynomials of \emph{arbitrary degree}.

Our method specializes a technique for certifying non-collision in full-dimensional volumes of configuration space presented in~\cite{dai2023certified}. By restricting ourselves to the planned motion case, we are able to leverage stronger results in optimization and dramatically reduce computation times.

We deploy our method to certify a piecewise-cubic motion plan for a 7-DOF arm, and a rapidly-exploring random tree (RRT) for a 12-DOF, bimanual manipulation example. As seen in Figure \ref{F:iiwa_piecwise_poly}, we demonstrate that our method is capable of discriminating between safe and unsafe motion plans which are arbitrarily close together and can be used to certify that a pair of geometries do not collide along a motion plan in at most hundreds of milliseconds. The result is a certifier which can verify the safety of arbitrarily complicated polynomial motion plans in a handful of seconds, or piecewise-linear plans from an RRT in a second.

\subsection{Notation}
Throughout, we use calligraphic letters ($\calX$) to denote sets, Roman capitals ($X$)
to denote matrices, and Roman lower case ($x$) to denote vectors or scalars. We use $X \succeq 0$ to denote that a symmetric matrix $X$ belongs to the cone of positive semidefinite (PSD) matrices which are denoted as $\setS_{+}$ .

For a vector $x$ of $m$ variables, we denote by $[x]_{d} \coloneqq \begin{bmatrix} 1, x_{1}, \dots, x_{m}, x_{1}^{2}, x_{1}x_{2}, \dots, x_{m}^{d}\end{bmatrix}^{T}$ the vector of all ${ {m + d} \choose d}$ monomials in $x_{1}, \dots, x_{m}$ of degree less than or equal to $d$. The vector space of real polynomials in $x$ of degree $d$ is written as $\setR[x]_{d}$ while the $n\times k$ matrices with entries in $\setR[x]_{d}$ are written as $\setR[x]^{n\times k}_{d}$. If $d$ is omitted, then $\setR[x]$ and $\setR[x]^{n\times k}$ are the set of all polynomials of arbitrary degree.

Finally, for $p(x) \in \setR[x]$ we denote by $\deg_{x_{i}}(p)$ the degree of the polynomial $p$ in the variable $x_{i}$ and similarly $\deg_{x_{i}}(P) = \displaystyle{\max_{i,j}}\ \deg(P_{i,j}(x))$ for $P(x) \in \setR[x]_{d}^{n \times k}$. When the variable is clear, the subscript is suppressed.

\section{Mathematical Preliminaries} \label{S: math prelims}
In this section, we introduce the essential mathematical ingredients which we will leverage to provide efficient certification of non-collision along our robot motion plans. Our method relies on well-known results from convex optimization \cite{boyd2004convex}, methods from sums-of-squares programming \cite[Chapter 3]{blekherman2012semidefinite}, and an algebraic reparametrization of the forward kinematics \cite{wampler2011numerical}.  

\subsection{Separating Convex Bodies} \label{S: sep bodies}
    \begin{figure}[htb]
        \centering
        
    \input{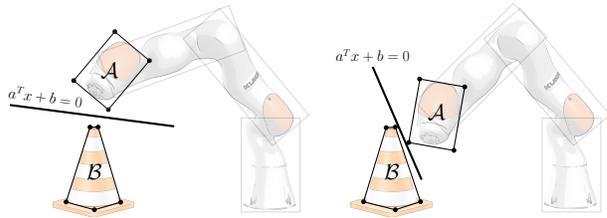}
        
        \caption{The closed, convex bodies $\calA$ and $\calB$ are collision-free if and only if there exists a hyperplane $a^{T}x + b = 0$ separating the two. As the bodies move in space, a different hyperplane may be needed to certify their non-collision.}
        \label{F: svm}
        \vspace{-0.2cm}
    \end{figure}
A well-known result from convex optimization is the Separating Hyperplane Theorem~\cite{boyd2004convex}, which states that two closed, convex bodies $\calA$ and $\calB$ do not intersect if and only there exists a hyperplane $a^{T}x + b = 0$ separating the bodies such as in Figure \ref{F: svm}. The search for such a hyperplane can be posed as the optimization program:
\begin{subequations} \label{E: sep hyperplane theorem}
\begin{gather}
    \find~ a, b~ \subjectto 
    \\
    a^{T}x + b > 0,~ \forall x \in \calA \label{E: sep hyperplane theorem cond A}
    \\ 
    a^{T}x + b < 0,~ \forall x \in \calB \label{E: sep hyperplane theorem cond B}
\end{gather}
\end{subequations}

which is convex since $\calA$ and $\calB$ are convex sets. The hyperplane $\calH = \{x \mid a^{T}x + b =0\}$ serves as a rigorous, mathematical proof that $\calA$ and $\calB$ do not intersect, and so we refer to $\calH$ as a \emph{separation certificate}.

\begin{example}\label{Ex: polytope condition}
    If $\calA$ is a polytope with vertices $\{v_{1}, \dots, v_{N}\}$, we can express \eqref{E: sep hyperplane theorem cond A} using $N$ linear constraints:
    \begin{align*}
        a^{T}v_{i} + b \geq 1 ~ \forall i \in \{1, \dots, N\}. 
    \end{align*}
\end{example}
\begin{example}\label{Ex: sphere condition}
    If $\calA$ is a sphere with center $c$ and radius $r$, we can express \eqref{E: sep hyperplane theorem cond A} using the PSD constraint:
    \begin{align*}
        \begin{bmatrix}
            (a^{T}c + b)I & ra \\
            ra^{T} & a^{T}c + b
        \end{bmatrix}
        \succeq 0
    \end{align*}
\end{example}

Further explicit examples of expressing \eqref{E: sep hyperplane theorem cond A} are available in~\cite{dai2023certified}.

\subsection{Polynomial Positivity on Intervals} \label{S: poly pos}
    Our method in Section \ref{S: non-collision on paths} will rely on a generalization of \eqref{E: sep hyperplane theorem} that was first introduced in~\cite{amice2022finding,dai2023certified}. This generalization relies heavily on the ability to certify that certain polynomials (resp. matrices) are non-negative (resp. positive semidefinite). Unlike in~\cite{amice2022finding,dai2023certified}, our polynomials will be \emph{univariate} and so we will be able to leverage much stronger theory than this prior work.

    We begin by recalling the definition of a Sums-of-Squares (SOS) polynomial and a SOS matrix~\cite{blekherman2012semidefinite}.
    \begin{definition}%
        A polynomial $p(x) \in \setR[x]_{2d}$ is \emph{sums-of-squares} if it can be expressed as $p(x) = \sum_{i=1}^{N} q_{i}^{2}(x)$ for $q_{i} \in \setR[x]_{d}$. Equivalently, $p(x)$ is SOS if it can be expressed as $[x]_{d}^{T}Q[x]_{d}$ for $Q \succeq 0$. The set of all SOS polynomials of degree $2d$ is denoted $\Sigma[x]_{2d}$.
    \end{definition}
    
    A similar notion exists for SOS matrices.
    \begin{definition}%
        A symmetric matrix $P(x) \in \setR[x]_{2d}^{n \times n}$ is a \emph{SOS matrix} if there exists a polynomial matrix $\Theta(x) \in \setR[x]_{d}^{k \times n}$ such that $P(x) = \Theta(x)^{T}\Theta(x)$. The set of all $n \times n$ SOS matrices of degree $2d$ is denoted $\Sigma[x]^{n \times n}_{2d}$.
    \end{definition}

    A useful characterization relating SOS matrices to SOS polynomials is given by the following theorem.
    \begin{theorem}[\cite{blekherman2012semidefinite}, Lemma 3.78]
        Let $P(x) \in \setR[x]_{2d}^{n \times n}$ be a symmetric polynomial matrix and let $y = \begin{bmatrix} y_{1}& \dots& y_{n} \end{bmatrix}$ be a vector of monomials. Define the scalar polynomial $p(x,y) = y^{T}P(x)y$. Then $P(x)$ is a SOS matrix if and only if $p(x,y)$ is SOS.
    \end{theorem}

    The characterization of SOS polynomials and matrices in terms of the existence of a semidefinite matrix $Q$ is attractive as it allows one to search for certificates of non-negativity using convex optimization, specifically semidefinite programming (SDP)~\cite{parrilo2000structured}. This is known as SOS programming.

    Finally, we recall two very strong theorems about the non-negativity of univariate polynomials and polynomial matrices on intervals.
    \begin{theorem}[Markov-Lucasz~\cite{roh2006discrete}]\label{T: poly pos on interval}
        A univariate polynomial $p(x)$ is non-negative on the non-empty interval $x \in [a,b]$ if and only if it can be expressed as
        \begin{align}\label{E: poly pos on interval}
            p(x) = 
            \begin{cases}
            \lambda(x) + (x-a)(b-x)\nu(x), & \deg(p)=2d
            \\
            (x-a)\lambda(x) + (b-x)\nu(x), & \deg(p)=2d+1
            \end{cases}
        \end{align}
        where $\lambda,~ \nu \in \bSigma[x]$.
        Moreover, if $d = \floor{\frac{\deg(p)}{2}}$ then
        \begin{align}\label{E: poly pos on interval deg conditions}
                \deg(\lambda) \leq 2d, ~~ 
                \deg(\nu) \leq 
                \begin{cases}
                    2d-2 & \text{if } \deg(p)=2d \\
                2d & \text{if } \deg(p)=2d+1
            \end{cases}
            .
        \end{align}
     \end{theorem}
     
     An analogous theorem holds in the univariate matrix case.
     \begin{theorem}\label{T: psd on interval}
        Let $P(x) \in \setR[x]^{m \times m}$ be a symmetric matrix of univariate polynomials. Then $P(x) \succeq 0$ on the non-empty interval $x \in [a,b]$ if and only if $p(x,y) = y^{T}P(x)y$ can be expressed as:
        \small{
        \begin{align}\label{E: psd on interval}
            p(x,y) =
            \begin{cases}
                \lambda(x,y) + (x-a)(b-x)\nu(x,y), & \deg(P)=2d
                \\
                (x-a)\lambda(x,y) + (b-x)\nu(x,y), & \deg(P)=2d+1
            \end{cases}
        \end{align}
        }
        where $\lambda,~ \nu \in \bSigma[x,y]$.
        Moreover, if $d = \floor{\frac{\deg(P)}{2}}$ then
        \begin{gather}\label{E: psd on interval deg conditions}
            \begin{gathered}
                \deg_{y_{i}}(\lambda) = 2,~ \deg_{y_{i}}(\nu) = 2, ~\deg_{x}(\lambda) \leq 2d
                    \\
                    \deg_{x}(\nu) \leq 
                    \begin{cases}
                        2d-2 & \text{if } \deg(P)=2d \\
                    2d & \text{if } \deg(P)=2d+1
                \end{cases}
            \end{gathered}
            .
        \end{gather}
    \end{theorem}
    \begin{proof}
        This follows by applying a similar argument as used in~\cite{roh2006discrete} to prove Theorem \ref{T: poly pos on interval} to the matrix $P(x)$ and leveraging the result of~\cite{choi1980real} that univariate PSD matrices are always SOS matrices.
    \end{proof}
    
    The polynomials $\lambda$ and $\nu$ in Theorems \ref{T: poly pos on interval} and \ref{T: psd on interval} are traditionally referred to as multiplier polynomials and serve as a \emph{certificate of non-negativity} on the interval $[a,b]$.

    \begin{remark}
        Analogs of Theorems \ref{T: poly pos on interval} and \ref{T: psd on interval} exist when the polynomials are allowed to be multivariate and are leveraged in~\cite{amice2022finding,dai2023certified}. The main advantage to restricting to the univariate case is the explicit degree bounds on the multiplier polynomials which are not available in the multivariate case.
    \end{remark}

\subsection{Algebraic Forward Kinematics} \label{S: alg kin}
    Our approach in Section \ref{S: non-collision on paths} will rely critically on a polynomial parametrization of the forward kinematics of a robot. 
    \begin{definition}
        A rigid-body robot is called \emph{algebraic} if all links are connected by a composition of the following two joints:
        \begin{itemize}
                \item Revolute (R): a 1-DOF joint permitting revolution about an axis. An example is a door hinge.
                \item Prismatic (P): a 1-DOF joint permitting translation along an axis. An example is a linear rail.
        \end{itemize}
    \end{definition}
    In addition to (R) and (P) joints, many common joints such as cylindrical, planar, and spherical joints can be represented as a composition of R and P joints \cite{wampler2011numerical}.
    
    In general, the forward kinematics of an algebraic robot can be expressed as a multilinear trigonometric polynomial~\cite{wampler2011numerical}. Concretely, using the monogram notation from~\cite{tedrakeManip}, the $w$\textsuperscript{th} component (for $w \in \{x,y,z\}$) of the position of a point $A$ relative to a frame $F$ and expressed in $F$ is can be written as:
    \begin{align}\label{E: natural forward kin}
        \leftidx{^{F}}p^{A}_{w}(q) = \sum_{j} c_{jw}\prod_{i} \xi_{ij,w}(q_{i}),
    \end{align}
    where $c_{jw}$ are constant coefficients, $\theta_{i} \coloneqq q_{i}$ with $\xi_{ij,w}(q_{i}) \in \{\cos(\theta_{i}), \sin(\theta_{i})\}$ if the $i$\textsuperscript{th} joint is associated to an (R) joint, and $z_{i} \coloneqq q_{i}$ with $\xi_{ij,w}(q_{i}) = z_{i}$ if the $i$\textsuperscript{th} joint is associated to a (P) joint. The collection of variables $q = \cup_{i} \{\theta_{i}, z_{i}\}$ are referred to as the configuration-space (C-space) variables.

    While \eqref{E: natural forward kin} is not a polynomial due to the presence of the trigonometric functions $\sin(\theta_{i})$ and $\cos(\theta_{i})$, it admits a rational reparametrization.

    \begin{definition}
        Define the substitution $\tau_{i} \coloneqq \tan\left(\frac{\theta_{i}}{2}\right)$.
        This substitution implies that 
        \begin{align*}
            \cos(\theta_{i}) = \frac{1-\tau_{i}^{2}}{1+\tau_{i}^{2}}, ~~~ \sin(\theta_{i}) = \frac{2\tau_{i}^{2}}{1+\tau_{i}^{2}}.
        \end{align*}
        The collection of variables $s = \cup_{i} \{\tau_{i}, z_{i}\}$ are referred to as the \emph{tangential-configuration space (TC-space) variables}.
    \end{definition}
    This substitution is known as the stereographic projection~\cite{spivakCalc} and is bijective for $\theta_{i} \in (-\pi, \pi)$. Under this substitution, \eqref{E: natural forward kin} can be expressed as an \emph{rational function} with a positive, polynomial denominator:
    \begin{gather}\label{E: rational forward kinematics gen}
        \leftidx{^F}p^{A}_{w}(s) =  \sum_{j} c_{jw} \prod_{i \in \calI_{F,A}} \frac{\leftidx{^F}f_{ij,w}^{A}(s_{i})}{\leftidx{^F}g_{ij,w}^{A}(s_{i})}  = 
            \frac{\leftidx{^F}f^{A}_{w}(s)}{\leftidx{^F}g^{A}_{w}(s)} 
        \end{gather}
        where $$\frac{\leftidx{^F}f_{ij,w}^{A}(s_{i})}{\leftidx{^F}g_{ij,w}^{A}(s_{i})} \in \left\{\frac{1-\tau_{i}^{2}}{1+\tau_{i}^{2}}, \frac{2\tau_{i}}{1+\tau_{i}^{2}}, \frac{z_{i}}{1}\right\}.$$
        
    \begin{example}
        We consider the forward kinematics of a pendulum mounted on a moving rail shown in Figure \ref{F: pend on rail}. Letting $z$ denote the commanded position from the left wall, and $\theta$ the angle from the vertical, the position $(p_{x}, p_{y})$ of the tip of the pendulum in C-space and TC-space coordinates is: 
        \bgroup
\def\arraystretch{2.2}%
        \begin{center}
        \begin{tabular}{c | c | c }
        & C-space & TC-space \\
        \hline \hline
        $p_{x}$ & $\displaystyle{l_{1}\sin(\theta) + z}$ & $\displaystyle{l_{1}\left(\frac{2\tau}{1+\tau^{2}}\right) + z}$
        \\
        \hline
        $p_{y}$ & $\displaystyle{l_{1}\cos(\theta)}$ & $\displaystyle{l_{1}\left(\frac{1-\tau^{2}}{1+\tau^{2}}\right)}$
        \end{tabular}
        \end{center}
                \egroup

    \end{example}
    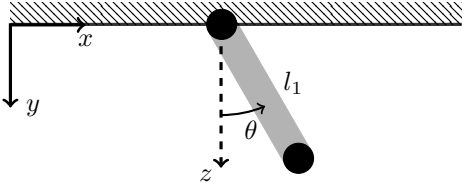
\begin{figure}[htb]
        \centering
        \begin{tikzpicture}
    \draw[rotate = 30,fill, opacity = 0.3]  (-3.8926,3.9595) rectangle (-3.504,1.8985);
    \draw[fill]  (-5.1915,1.6064) ellipse (0.2 and 0.2);
    \draw[fill]  (-4.17,-0.175) ellipse (0.2 and 0.2);
    \draw[very thick, <->] (-8,0.5) -- (-8,1.6) -- (-7.0,1.6);
    \draw[very thick, dashed, <->] (-5.2,-0.3) -- (-5.2,1.6);
    \node at (-7.0,1.4) {$x$};
    \node at (-7.7,0.5) {$y$};
    \node at (-5.4,-0.4) {$z$};
    
    \draw[->,thick] (-5.2,0.4) arc (-90:-67:1.5);
    
    \node at (-4.8,0.2) {$\theta$};
    \node at (-4.2388,0.8241) {$l_1$};

\draw[EDB] (-8,1.6064) -- (-2, 1.6064) -- (-2,1.9) -- (-8,1.9) --cycle;
\draw[very thick, opacity = 0.8]  plot[smooth] coordinates {
    (-8,1.6064) 
    (-2, 1.6064)
};
    
\end{tikzpicture}
        \caption{The pendulum on rail system.}\label{F: pend on rail}
    \end{figure}

    Notice that the stereographic projection induces some warping of distances in configuration space. However, in the limited joint range of many standard industrial manipulators, this warping is minor and straight lines in TC-space are approximately straight lines in C-space.

    \begin{remark}
        An alternative reparametrization of \eqref{E: natural forward kin} introduces the substitutions $\sigma_{i} \coloneqq \sin(\theta_{i})$ and $c_{i} \coloneqq \cos(\theta_{i})$ with the constraint that $\sigma_{i}^{2} + c_{i}^{2} = 1$ and is known as the algebraic-configuration space (AC-space). All results in this paper are easily extended to AC-space with no change in the underlying mathematics. We prefer to use TC-space throughout as it is easier to generate feasible motion plans.
        
    \end{remark}

\section{Problem Formulation}
    We consider an algebraic, rigid-body robot operating in a known environment. The geometry of the robot and all obstacles in the environment are assumed to have a known decomposition as a union of compact, convex bodies such as spheres, capsules, cylinders, or polytopes. Such collision geometries of our task space are readily available through standard tools such as V-HACD~\cite{mamou2009simple} and are often a required step for simulating any given environment. We refer to the pairs of bodies which can collide as the collision pairs. Under these assumptions, we define the following problem.
    \begin{problem}{\cert}\label{P: cert problem}
        Given an arbitrary polynomial motion plan in TC-space $\rho(t): [0,1] \mapsto s$, \textbf{certify} that the plan contains no collisions. Formally, we seek an algorithm with no false optimism; The algorithm answers \safe only if $\rho$ contains no collisions.
    \end{problem}
    Notice that solving \cert also covers the case when the plan $\rho$ is a piecewise polynomial, as we can simply certify each piecewise segment individually.

    \begin{remark}   
        Finitely sampling for collisions along the plan is an algorithm with false optimism. It can return \safe when the plan does in fact contain collisions. We prefer an algorithm which only declares \safe when the plan is in fact safe, but may sometimes declare \nsafe even if the plan is \safe. 
    \end{remark}

    \section{Proving Non-collision along a Plan} \label{S: non-collision on paths}
    In this section, we provide an efficient solution to \cert based on convex optimization, specifically SOS programming. Our method is based on generalizing \eqref{E: sep hyperplane theorem} to handle a motion plan $\rho(t)$ of configurations.

    Let $\calA(s)$ denote the position of the convex body $\calA$ at the TC-space configuration $s$. Recall from Section \ref{S: alg kin} that we can express any point $x(s) \in \calA(s)$ as a rational function with a positive denominator and so $x(\rho(t))$ for $t \in [0,1]$ is again a rational function with a positive denominator (rational functions are closed under composition with polynomials). 

    Now, as the position of two bodies $\calA(t) \coloneqq \calA(\rho(t))$ and $\calB(t) \coloneqq \calB(\rho(t))$ vary over the motion plan, a single, static hyperplane may be insufficient to certify that $\calA(t)$ and $\calB(t)$ do not collide as in Figure \ref{F: svm}. Therefore, we allow the hyperplane parameters $a$ and $b$ to also vary as a polynomial function of $t$. 

    For every pair of bodies $(\calA(t), \calB(t))$ which can collide in the environment, we search for a polynomially parametrized family of hyperplanes $(a_{\calA,\calB}(t), b_{\calA,\calB}(t))$ which certify that $(\calA(t), \calB(t))$ do not collide for $t \in [0,1]$. Concretely:

    \begin{subequations}\label{E: cert by hyperplane poly}
        \begin{gather} 
            {\forall~ \text{pairs } \calA,\calB} \nonumber
            \\
            ~\find a_{\calA,\calB}(t), b_{\calA,\calB}(t)
            ~~\textbf{subject to}
            \\
            \begin{rcases}
            a^{T}_{\calA, \calB}(t)x(t) + b_{\calA, \calB}(t) > 0, ~\forall x \in \calA(t)  %
            \\
            a^{T}_{\calA, \calB}(t)x(t) + b_{\calA, \calB}(t) < 0, ~\forall x \in \calB(t) %
            \end{rcases}
            \forall t \in [0,1].
            \label{E: cert by hyperplane sep cond}
        \end{gather}
        \end{subequations}
    This is an optimization program over polynomials which can be solved using SDP by transforming \eqref{E: cert by hyperplane sep cond} into equivalent linear and semidefinite constraints using Theorems \ref{T: poly pos on interval} and \ref{T: psd on interval}. A feasible solution to \eqref{E: cert by hyperplane poly} is a collection of polynomials
    \begin{align*}
    \calC = \bigcup_{\calA, \calB}\left\{a_{\calA, \calB}, b_{\calA, \calB}, \lambda_{\calA}, \nu_{\calA}, \lambda_{\calB}, \nu_{\calB}, \right\}
    \end{align*}
    and is a \emph{certificate of non-collision}. At every time point $t \in [0,1]$, the polynomials $(a_{\calA, \calB}(t), b_{\calA, \calB}(t))$ serve as a \emph{separation certificate} for $(\calA(t), \calB(t))$. Meanwhile $\{\lambda_{\calA}, \nu_{\calA}, \lambda_{\calB}, \nu_{\calB}\}$ serve as a \emph{certificate of positivity} that $(a_{\calA, \calB}(t), b_{\calA, \calB}(t))$ are a separation certificate for every ${t \in [0,1]}$.

    \begin{example}\label{Ex: polytope condition rational}
        If $\calA(t)$ is a polytope with vertices $v_{i}$ at positions $\leftidx{^{F}}p^{v_{i}}(t) = \frac{\leftidx{^F}f^{v_{i}}(\rho(t))}{\leftidx{^F}g^{v_{i}}(\rho(t))}$ for $i \in \{1,\dots,N\}$, then we can enforce \eqref{E: cert by hyperplane sep cond} using $N$ polynomial constraints:
        \begin{align*}
            a^{T}(t)\leftidx{^F}f^{v_{i}}(\rho(t)) + (b(t)-1)\leftidx{^F}g^{v_{i}}(\rho(t)) \geq 0.
        \end{align*}
        This constraint can be enforced using Theorem \ref{T: poly pos on interval}.
    \end{example}
    \begin{example}\label{Ex: sphere condition rational}
        If $\calA$ is a sphere with center $c$ at position $\leftidx{^{F}}p^{c}(t) = \frac{\leftidx{^F}f^{c}(\rho(t))}{\leftidx{^F}g^{c}(\rho(t))}$ and radius $r$, we can express \eqref{E: cert by hyperplane sep cond} using the SOS matrix constraint:
        \begin{align*}
            \small{
            \begin{bmatrix}
                h(t)I & r\left(\leftidx{^F}g^{c}(\rho(t))\right)a(t) \\
                r\left(\leftidx{^F}g^{c}(\rho(t))\right)a^{T}(t) & h(t)
            \end{bmatrix}
            \succeq 0
            }
        \end{align*}
        where $h(t) = a^{T}(t)\leftidx{^F}f^{c}(\rho(t)) + b(t)\leftidx{^F}g^{c}(\rho(t))$. Such a positivity constraint can be enforced using Theorem \ref{T: psd on interval}.
    \end{example}

    \begin{remark}
        Solving \eqref{E: cert by hyperplane poly} requires choosing a finite degree basis for all polynomials. As Theorems \ref{T: poly pos on interval} and \ref{T: psd on interval} specify the necessary and sufficient degree for the univariate multiplier polynomials, the only hyperparameter in \eqref{E: cert by hyperplane poly} is the degree of the hyperplanes. 
        This is in contrast to~\cite{dai2023certified,amice2022finding} where both the degree of the hyperplane and the degree of the multipliers are difficult-to-choose hyperparameters due to the multivariate nature of the problem.
    \end{remark}

\section{Results}

In this section, we evaluate the effectiveness of the proposed trajectory certification program \eqref{E: cert by hyperplane poly} for solving  \cert. An open source implementation of our method is \href{https://github.com/AlexandreAmice/drake/tree/CspaceFreePathFeature}{publicly available}\footnote{\url{https://github.com/AlexandreAmice/drake/tree/CspaceFreePathFeature}} and undergoing code review in Drake~\cite{drake}. All computations are performed on a laptop computer with an 11th Generation Intel I9 CPU and 64GB of RAM.

Each segment is certified by running one instance of \eqref{E: cert by hyperplane poly} for \emph{each collision pair} to understand the performance of the certification scheme in various regimes of the state space. In practical settings, it is possible to rapidly cull some of the collision pairs which cannot collide over the course of the motion by using, for example, bounding-volume hierarchies \cite{pan2012fcl, redon2005fast} which would improve the runtime in our experiments.

\subsection{Certifying an RRT for A Bimanual Manipulator}
\begin{figure}[tb]
\centering
    \includegraphics[width = 0.9\columnwidth, trim = {0cm 0cm 0cm 0.5cm},clip]{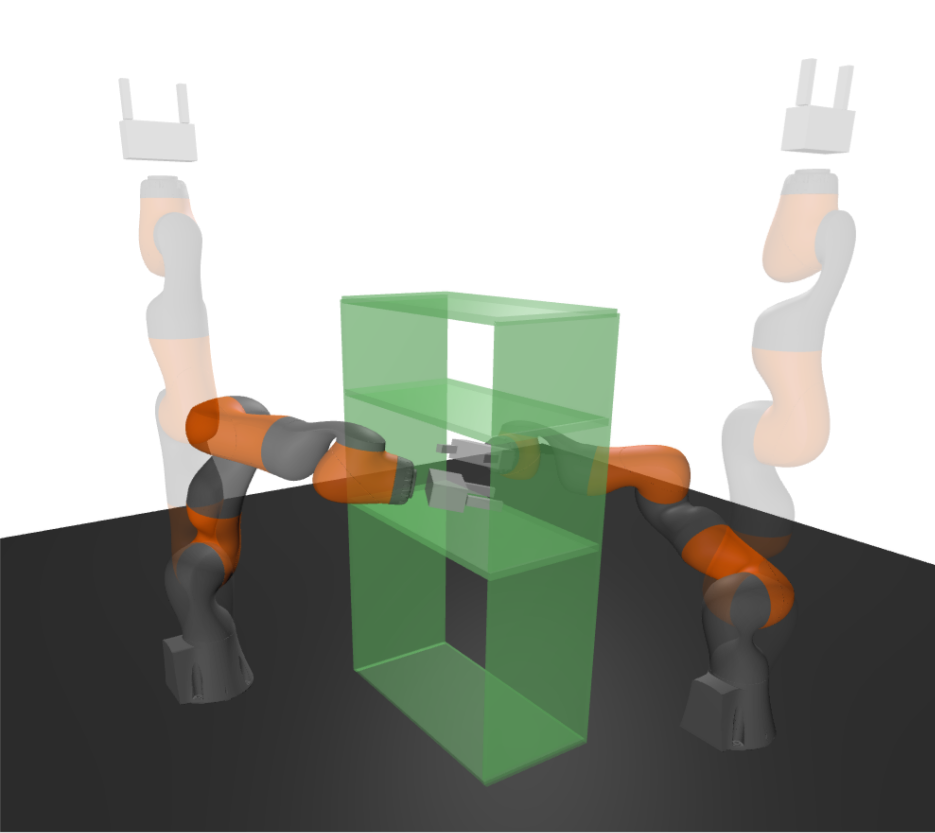}
    \caption{A pair of Kuka iiwa reaching from the straight up start configuration (translucent arms) into the shelf (opaque arms). The close confines of this motion plan make checking safety via finite sampling challenging. An animation of the motion plan is available 
    \href{https://alexandreamice.github.io/project/c-iris-path/14dof_hyperplanes.html}{here.}
    \vspace{-0.5cm}
    }\label{F: bimanul example}
\end{figure}

We consider the task of certifying a motion plan for a 12 degree of freedom (DOF) bimanual KUKA iiwa reaching into the shelf. The start and end configurations are shown in Figure \ref{F: bimanul example}. The environment contains 246 total collision pairs.

An RRT is grown in TC-space, with edges extended by sampling one hundred intermediate points between tree nodes. We grow the RRT until the task and goal are connected and the motion plan between them is certified as \safe by \eqref{E: cert by hyperplane poly}. This requires 105 edges and so a total of 25,830 instances of \eqref{E: cert by hyperplane poly} are solved with the hyperplane parameterized by a linear polynomial. All programs are solved in parallel.

In Figure \ref{F: bimanual timing carpet}, we plot the entire distribution of solve times for the 25,830 programs. We group the collision geometry pairs by the highest degree condition given by \eqref{E: cert by hyperplane poly}. The most expensive program takes 297 ms to solve with Mosek~\cite{mosek}, while the least expensive program takes 0.66 ms to solve. We note that Figure \ref{F: bimanual timing carpet} is plotted on a log scale and demonstrates approximately quadratic growth as the degree of the polynomials increases to the right. On the other hand, almost no pattern is observed in the length of time taken to certify an individual edge for a given pair.

In Table \ref{T: bimanaul stats}, we provide some aggregated statistics on the certification procedure. We recall that an edge of the RRT is considered safe if \eqref{E: cert by hyperplane poly} is feasible for \emph{all} 246 collision pairs. Due to the very close proximity of the collision geometries, we do not expect every edge to be collision-free. Indeed, we see that just under 30\% of the edges in our tree are certified as collision free.

 We observe that many of the uncertified edges correspond to motions in the tight confines of the shelf and so may in fact contain collisions. Therefore, if \eqref{E: cert by hyperplane poly} for a given edge is not feasible for every collision pair, then we resample $1e5$, uniformly spaced configurations along that edge. If a collision is found, we confirm that the edge is \nsafe. Otherwise, we cannot confirm or deny the safety of the edge. We see that in 96\% of cases, more dense sampling recovers a collision, which verifies that the infeasibility of \eqref{E: cert by hyperplane poly} for these edges is not due to choosing too low a degree for the hyperplane. On the other hand, exactly two edges are neither declared \safe, nor confirmed \nsafe. These two examples correspond to motions which undergo very large task space displacement of all links, but do not appear to contain collisions. It is likely the case that a higher degree polynomial could certify these edges.

\begin{figure*}[htb] \label{F: bimanual stats}
    \begin{subfigure}[c]{0.45\textwidth}
    \includegraphics[width = \columnwidth,trim = {0.25cm 0.25cm 0.3cm 0.25cm},clip]{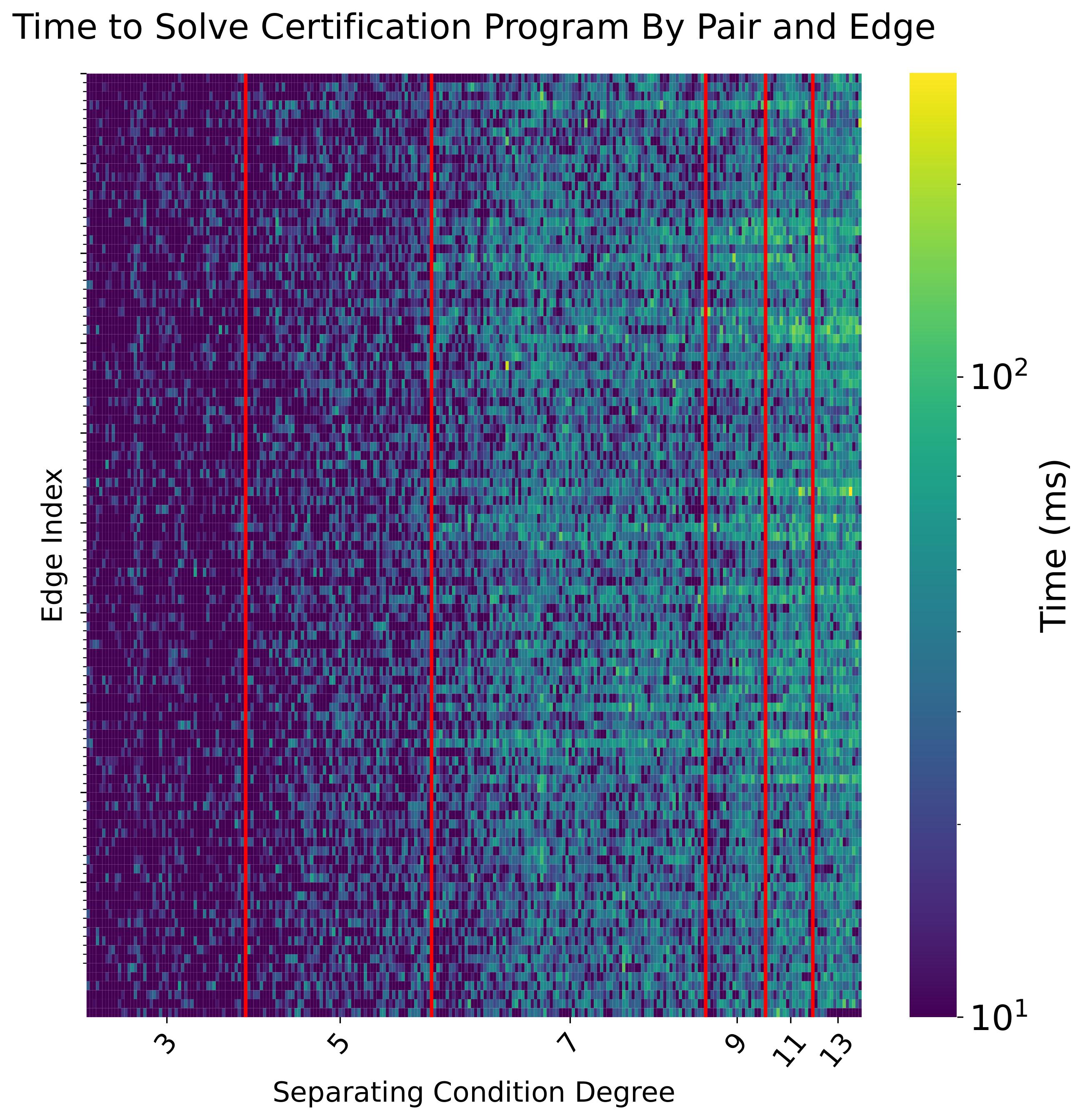}
    \subcaption{The time required to certify each edge and each collision pair in an RRT for the bimanual iiwa system. The y-axis is the index of the edge, and the x-axis is the index of the collision pair, grouped by the degree of the separating hyperplane condition \eqref{E: cert by hyperplane sep cond}. The time taken to solve the certification program for a collision pair scales quadratically in the degree of the separating condition \eqref{E: cert by hyperplane sep cond}
    }\label{F: bimanual timing carpet}
    \end{subfigure}
    \hfill
    \begin{subtable}[c]{0.50\textwidth}
       \centering
        \begin{tblr}{Q[6cm,halign=c,valign=m] |
        Q[2cm,halign=c,valign=m] }
        \SetCell[c=2]{c}{\bfseries Aggregate Statistics For Certifying the Entire RRT}
        \\
        \hline \hline
                \# Instances of \eqref{E: cert by hyperplane poly} Solved & 25830 \\
                \hline
        \# Of Collision Pairs & 246 \\
        \hline
        \# Of Edges & 105 \\
        \hline
        \# Edges \safe & 31
        \\
        \hline
        \# Edges Confirmed \nsafe & 72
        \\
        \hline
        \# Edges Unconfirmed \nsafe & 2
        \\
        \hline
        Average Solver Time to Certify an Edge & 299.9 msec. \\
        \hline
        Total Solver Time to Certify Goal Plan & 1.211 sec.
        \\
        \hline
        Total Solver Time to Certify RRT & 31.5 sec. \\
       \end{tblr}
       \vspace{1.8cm}
       \subcaption{Aggregated statistics for certifying an RRT for the bimanual iiwa system. An edge is \safe if all 246 instances of \eqref{E: cert by hyperplane poly} are feasible.  If \eqref{E: cert by hyperplane poly} is infeasible, the corresponding edge is densely sampled with $1e5$ points to find a collision. If a collision is found, the edge is confirmed \nsafe. We see that in two cases, neither \eqref{E: cert by hyperplane poly} is feasible for all pairs, nor is a collision found. The final RRT plan is certified as safe in just over 1 second, and the whole tree is certified in about 30 seconds.
       }
       \label{T: bimanaul stats}
    \end{subtable}
\caption{
Timing statistics and aggregated statistics for certifying an RRT for the bimanual iiwa system from Figure \ref{F: bimanul example}. An instance of \eqref{E: cert by hyperplane poly} is solved for all 246 collision pairs and all 105 edges in just over $30$ seconds.
}
\end{figure*}

\subsection{Certifying Cubic Polynomial Plans for a 7-DOF iiwa}
We demonstrate our method's ability to certify higher-order motion plans. We consider a 7-DOF Kuka iiwa interacting with a pair of shelves shown in Figure \ref{F:iiwa_piecwise_poly} moving along a piecewise polynomial plan with thirty pieces. Each piece is parametrized as a cubic, Hermite spline

We consider two such motion plans. The blue plan in Figure \ref{F:iiwa_piecwise_poly} is collision free, while exactly two of the thirty pieces of the red plan contain minor collisions with the shelf. We solve \eqref{E: cert by hyperplane poly} with a linearly parametrized hyperplane for each segment of the plan. 

In the case of the blue curve, all thirty pieces of the motion plan are certified as \safe in 8.99 seconds. Additionally, the twenty-eight pieces of the red plan which are safe are marked as \safe. Meanwhile, the optimizer reported that \eqref{E: cert by hyperplane poly} is infeasible for the two unsafe segments. The total time to solve the programs for unsafe plan was 9.21 seconds, which can be reduced to 6.91 seconds if the two unsafe segments are allowed to terminate as soon as one collision pair fails to certify its safety.

 This demonstrates the precision of our method, which is capable of discriminating the safety of two visually indistinguishable motion plans which differ by at most 20mm.

\section{Discussion and Conclusion}
We present a method based on SOS programming for certifying non-collision along piecewise polynomial motion plans of arbitrary degree for kinematic systems composed of algebraic joints. These systems include the majority of open kinematic chain robots currently available. The proposed method provides fully rigorous certificates of non-colllision for individual linear motion plans in milliseconds even for high DOF systems and can certify more complicated plans in a handful of seconds. 

Though currently not fast enough to fully replace a more rapid, sampling-based collision checker as a subroutine in a sampling-based motion planner, we anticipate that further efficiency improvements are possible. As already noted, the number of instances of \eqref{E: cert by hyperplane poly} can be dramatically reduced by pruning collision pairs which cannot possibly collide on the motion plan.
Secondly, all problems are solved using the general-purpose SDP solver Mosek. The formulation \eqref{E: cert by hyperplane poly} has substantial structure for which accelerated solution methods exist \cite{yuan2022polynomial}.

The primary drawback to the adoption of our method is the requirement that the motion plans be algebraic, i.e. they must be given in a polynomial reparameterization of the configuration space such as the TC-space or the AC-space forcing the practioner to change their choice of configuration-space variable. In the case of standard randomized motion plans such as PRMs and RRTs, this corresponds to simply changing the notion of distance and linear interpolation in the extend step of these methods.

While this may be a low barrier for purely kinematic motion plans, it may be more challenging for dynamic plans. The TC-space reparametrization greatly distorts configuration space near the joint angle $\theta_{i} = \pi$. Meanwhile the AC-space requires generating plans which conform to a strict constraint manifold. 

Nevertheless, the proposed method is suitable for real world application, providing fully rigorous certificates of non-collision with provable correctness.

\section{Acknowledgements}
The authors are grateful to Hongkai Dai for his help in improving both the quality of this manuscript and the numerical implementation of this paper.

\clearpage
\IEEEtriggeratref{16}

\bibliography{bibliographies/refs} 
\bibliographystyle{IEEEtran}

\end{document}